\documentclass[a4paper, 12pt]{article}
\usepackage[utf8]{inputenc}
\usepackage[T1]{fontenc}
\usepackage{amsmath}
\usepackage{bm}
\usepackage{amsfonts}
\usepackage{multirow}
\usepackage[exponent-product=\cdot]{siunitx}
\usepackage{hyperref}
\usepackage{gensymb}
\usepackage{textcomp}

\usepackage{graphicx}

\newcommand{\labeleq}[1]{\label{equation:#1}}
\newcommand{\refeq}[1]{Eqn.~(\ref{equation:#1})}
\newcommand{\refeqb}[1]{(\ref{equation:#1})}
\newcommand{\labelfig}[1]{\label{figure:#1}}
\newcommand{\reffig}[1]{Fig.~\ref{figure:#1}}

\newcommand{\labeltab}[1]{\label{table:#1}}
\newcommand{\reftab}[1]{Tab.~\ref{table:#1}}
\newcommand{\labelsec}[1]{\label{section:#1}}
\newcommand{\refsec}[1]{Section~\ref{section:#1}}
\newcommand{\reflis}[1]{Listing~\ref{listing:#1}}

\newcommand{\R}{\mathbb{R}}

\newcommand{\hide}[1]{}

\usepackage{amsthm}
\newtheoremstyle{definitionStyle}
  {}
  {}
  {}
  {}
  {\bfseries}
  {.}
  { }
  {\thmname{#1}\thmnumber{ #2}\thmnote{ (#3)}}
\theoremstyle{definitionStyle}
\newtheorem{definition}{Definition}
\newtheorem{theorem}[definition]{Theorem}

\usepackage{listings}
\usepackage{xcolor}
\usepackage{caption}
\makeatletter
\def\lst@numbersymbol{}
\lst@Key{numbersymbol}{}{\def\lst@numbersymbol{#1}}
\def\lst@labellis{}
\lst@Key{labellis}{}{\def\lst@label{listing:#1}}
\makeatother
\lstdefinelanguage{maple}{
    morekeywords={Basis, tdeg, op, indets, ListTools, SearchAll, map, degree, evalb},
    sensitive=false,
    morecomment=[l]{\#},
    morestring=[b]"
}

\makeatletter
\providecommand\color[2][]{%
  \GenericError{(gnuplot) \space\space\space\@spaces}{%
    Package color not loaded in conjunction with
    terminal option `colourtext'%
  }{See the gnuplot documentation for explanation.%
  }{Either use 'blacktext' in gnuplot or load the package
    color.sty in LaTeX.}%
  \renewcommand\color[2][]{}%
}%
\providecommand\includegraphics[2][]{%
  \GenericError{(gnuplot) \space\space\space\@spaces}{%
    Package graphicx or graphics not loaded%
  }{See the gnuplot documentation for explanation.%
  }{The gnuplot epslatex terminal needs graphicx.sty or graphics.sty.}%
  \renewcommand\includegraphics[2][]{}%
}%
\providecommand\rotatebox[2]{#2}%
\@ifundefined{ifGPcolor}{%
  \newif\ifGPcolor
  \GPcolorfalse
}{}%
\@ifundefined{ifGPblacktext}{%
  \newif\ifGPblacktext
  \GPblacktexttrue
}{}%
\let\gplgaddtomacro\g@addto@macro
\gdef\gplbacktext{}%
\gdef\gplfronttext{}%
\makeatother

\title{\LARGE \bf
Globally Optimal Solution to Inverse Kinematics of 7DOF Serial Manipulator
}
\author{
  Pavel Trutman\\
  CIIRC CTU in Prague
  \and
  Mohab Safey El Din\\
  Sorbonne Universit\'e, LIP6 CNRS
  \and
  Didier Henrion\\
  LAAS-CNRS, FEE CTU in Prague
  \and
  Tomas Pajdla\\
  CIIRC CTU in Prague
}
\begin{document}
\maketitle
\begin{abstract}
  The Inverse Kinematics (IK) problem is to find robot control parameters to bring it into the desired position under the kinematics and collision constraints. We present a global solution to the optimal IK problem for a general serial 7DOF manipulator with revolute joints and a quadratic polynomial objective function. We show that the kinematic constraints due to rotations can all be generated by second-degree polynomials. This is important since it significantly simplifies further step where we find the optimal solution by Lasserre relaxations of non-convex polynomial systems. We demonstrate that the second relaxation is sufficient to solve the 7DOF IK problem. Our approach is certifiably globally optimal. We demonstrate the method on the 7DOF KUKA LBR IIWA manipulator and show that we are able to compute the optimal IK or certify in-feasibility in $99~\%$ tested poses.
  \end{abstract}
  \section{Introduction}
  \noindent The Inverse Kinematics (IK) problem is one of the most important problems in robotics~\cite{shigley1980theory}. The problem is to find robot control parameters to bring it into the desired position under the kinematics and collision constraints~\cite{Jazar-2007}. 

  The IK problem has been extensively studied in robotics and control~\cite{Raghavan1993InverseKO,Raghavan1995SolvingPS}. The classical formulation~\cite{Raghavan1993InverseKO} of the problem for 6 degrees of freedom (6DOF) serial manipulators leads to solving systems of polynomial equations~\cite{Cox-IVA-2015,Wampler1990NumericalCM}. This is in general hard (``EXPSPACE complete''~\cite{MAYR1982305}) algebraic computational problem, but practical solving methods have been developed for 6DOF manipulators~\cite{Raghavan1993InverseKO,DBLP:journals/trob/ManochaC94,Diankov:2010:ACR:2125842}. 

  An important generalization of the IK problem aims at finding the optimal control parameters for an under-constrained mechanism, i.e.\  when the number of controlled joints in a manipulator is larger than six. Then, an algebraic computation problem turns into an optimization problem over an algebraic variety~\cite{Cox-IVA-2015} of possible IK solutions. It is particularly convenient to choose a polynomial objective function to arrive at a semi-algebraic optimization problem~\cite{Lasserre2001}. 

  Semi-algebraic optimization problems are in general non-convex but can be solved with certified global optimality~\cite{Lasserre2015} using the  Lasserre hierarchy  of  convex  optimization problems~\cite{Lasserre2001}. Computationally, however, semi-algebraic optimization problems are in general extremely hard 
  and were often considered too expensive to be used in practice. In this paper, we show that using ``algebraic pre-processing'' in semi-algebraic optimization methods becomes practical in solving the IK problem of general 7DOF serial manipulators with a polynomial objective function.
\begin{figure}[t]
  \includegraphics[width=0.49\textwidth]{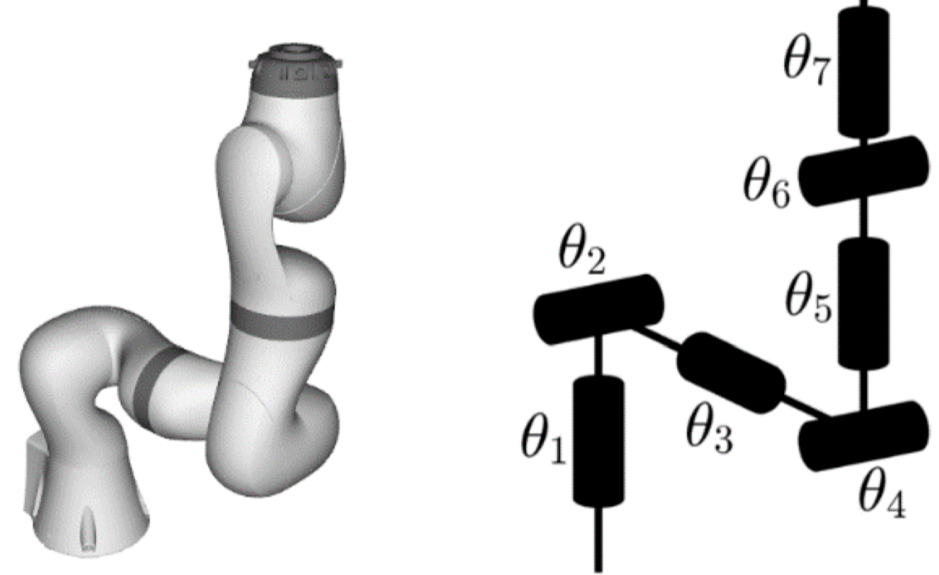} \resizebox{0.49\textwidth}{!}{
\begingroup
  \makeatletter
  \gdef\gplbacktext{}%
  \gdef\gplfronttext{}%
  \makeatother
  \ifGPblacktext
    \def\colorrgb#1{}%
    \def\colorgray#1{}%
  \else
    \ifGPcolor
      \def\colorrgb#1{\color[rgb]{#1}}%
      \def\colorgray#1{\color[gray]{#1}}%
      \expandafter\def\csname LTw\endcsname{\color{white}}%
      \expandafter\def\csname LTb\endcsname{\color{black}}%
      \expandafter\def\csname LTa\endcsname{\color{black}}%
      \expandafter\def\csname LT0\endcsname{\color[rgb]{1,0,0}}%
      \expandafter\def\csname LT1\endcsname{\color[rgb]{0,1,0}}%
      \expandafter\def\csname LT2\endcsname{\color[rgb]{0,0,1}}%
      \expandafter\def\csname LT3\endcsname{\color[rgb]{1,0,1}}%
      \expandafter\def\csname LT4\endcsname{\color[rgb]{0,1,1}}%
      \expandafter\def\csname LT5\endcsname{\color[rgb]{1,1,0}}%
      \expandafter\def\csname LT6\endcsname{\color[rgb]{0,0,0}}%
      \expandafter\def\csname LT7\endcsname{\color[rgb]{1,0.3,0}}%
      \expandafter\def\csname LT8\endcsname{\color[rgb]{0.5,0.5,0.5}}%
    \else
      \def\colorrgb#1{\color{black}}%
      \def\colorgray#1{\color[gray]{#1}}%
      \expandafter\def\csname LTw\endcsname{\color{white}}%
      \expandafter\def\csname LTb\endcsname{\color{black}}%
      \expandafter\def\csname LTa\endcsname{\color{black}}%
      \expandafter\def\csname LT0\endcsname{\color{black}}%
      \expandafter\def\csname LT1\endcsname{\color{black}}%
      \expandafter\def\csname LT2\endcsname{\color{black}}%
      \expandafter\def\csname LT3\endcsname{\color{black}}%
      \expandafter\def\csname LT4\endcsname{\color{black}}%
      \expandafter\def\csname LT5\endcsname{\color{black}}%
      \expandafter\def\csname LT6\endcsname{\color{black}}%
      \expandafter\def\csname LT7\endcsname{\color{black}}%
      \expandafter\def\csname LT8\endcsname{\color{black}}%
    \fi
  \fi
    \setlength{\unitlength}{0.0500bp}%
    \ifx\gptboxheight\undefined%
      \newlength{\gptboxheight}%
      \newlength{\gptboxwidth}%
      \newsavebox{\gptboxtext}%
    \fi%
    \setlength{\fboxrule}{0.5pt}%
    \setlength{\fboxsep}{1pt}%
\begin{picture}(5760.00,4320.00)%
    \gplgaddtomacro\gplbacktext{%
      \csname LTb\endcsname
      \put(949,726){\makebox(0,0){\strut{}-800}}%
      \csname LTb\endcsname
      \put(1321,659){\makebox(0,0){\strut{}-600}}%
      \csname LTb\endcsname
      \put(1693,591){\makebox(0,0){\strut{}-400}}%
      \csname LTb\endcsname
      \put(2065,523){\makebox(0,0){\strut{}-200}}%
      \csname LTb\endcsname
      \put(2438,456){\makebox(0,0){\strut{}0}}%
      \csname LTb\endcsname
      \put(2810,388){\makebox(0,0){\strut{}200}}%
      \csname LTb\endcsname
      \put(3181,320){\makebox(0,0){\strut{}400}}%
      \csname LTb\endcsname
      \put(3553,252){\makebox(0,0){\strut{}600}}%
      \csname LTb\endcsname
      \put(3925,185){\makebox(0,0){\strut{}800}}%
      \csname LTb\endcsname
      \put(4094,250){\makebox(0,0)[l]{\strut{}-800}}%
      \csname LTb\endcsname
      \put(4230,436){\makebox(0,0)[l]{\strut{}-600}}%
      \csname LTb\endcsname
      \put(4365,623){\makebox(0,0)[l]{\strut{}-400}}%
      \csname LTb\endcsname
      \put(4501,809){\makebox(0,0)[l]{\strut{}-200}}%
      \csname LTb\endcsname
      \put(4636,995){\makebox(0,0)[l]{\strut{}0}}%
      \csname LTb\endcsname
      \put(4771,1181){\makebox(0,0)[l]{\strut{}200}}%
      \csname LTb\endcsname
      \put(4907,1367){\makebox(0,0)[l]{\strut{}400}}%
      \csname LTb\endcsname
      \put(5042,1553){\makebox(0,0)[l]{\strut{}600}}%
      \csname LTb\endcsname
      \put(5178,1739){\makebox(0,0)[l]{\strut{}800}}%
      \put(868,830){\makebox(0,0)[r]{\strut{}0}}%
      \put(868,1001){\makebox(0,0)[r]{\strut{}100}}%
      \put(868,1173){\makebox(0,0)[r]{\strut{}200}}%
      \put(868,1344){\makebox(0,0)[r]{\strut{}300}}%
      \put(868,1516){\makebox(0,0)[r]{\strut{}400}}%
      \put(868,1687){\makebox(0,0)[r]{\strut{}500}}%
      \put(868,1859){\makebox(0,0)[r]{\strut{}600}}%
      \put(868,2030){\makebox(0,0)[r]{\strut{}700}}%
      \put(868,2201){\makebox(0,0)[r]{\strut{}800}}%
      \put(868,2372){\makebox(0,0)[r]{\strut{}900}}%
      \put(868,2544){\makebox(0,0)[r]{\strut{}1000}}%
      \put(3024,3876){\makebox(0,0){\strut{}}}%
    }%
    \gplgaddtomacro\gplfronttext{%
      \csname LTb\endcsname
      \put(2267,263){\makebox(0,0){\strut{}$x$ [mm]}}%
      \put(5104,925){\makebox(0,0){\strut{}$y$ [mm]}}%
      \put(70,1687){\makebox(0,0){\strut{}$z$ [mm]}}%
    }%
    \gplbacktext
    \put(0,0){\includegraphics{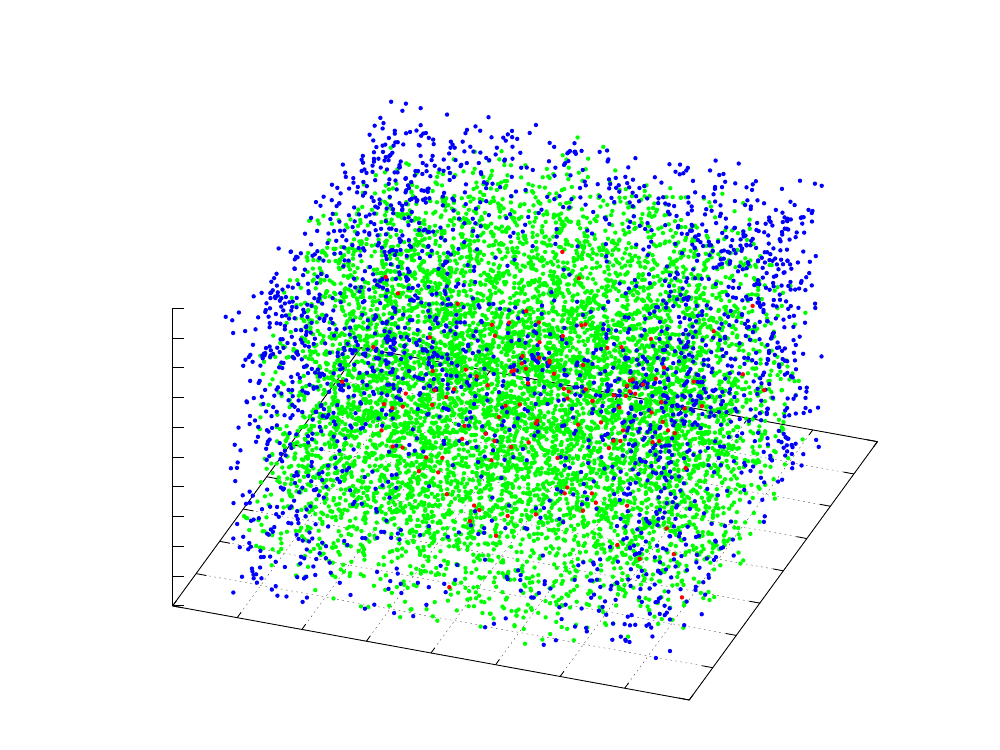}}%
    \gplfronttext
  \end{picture}%
\endgroup
  }
  \caption{(left) 7DOF serial manipulator (KUKA LBR IIWA), and (middle) its kinematic model~\cite{Kuhlemann2016}. (right) We can optimally solve its inverse kinematics (green) or find it infeasible (blue) in $99~\%$ of \num{10000} tested poses.}
  \labelfig{teaser}
\end{figure}
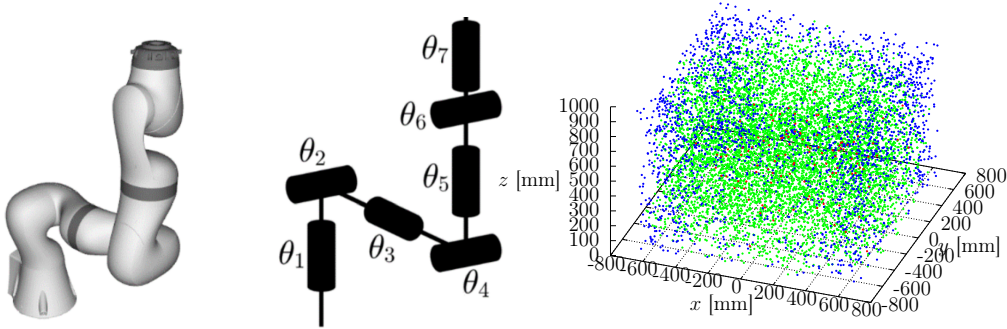
\subsection{Contribution}
\noindent Our main contributions are:

\noindent (1) We prove that the variety of IK solutions of all generic 7DOF revolute serial manipulators can be generated by second-degree polynomials only (Theorem~\ref{THM}). This considerably reduces the complexity of semi-algebraic optimization and makes it computationally feasible.

\noindent (2) We provide a method for computing a globally optimal solution to the IK problem for a general 7DOF serial manipulator with a polynomial objective function.

\noindent (3) We demonstrate that our approach works on a practical 7DOF KUKA LBR IIWA manipulator and allows us to solve $99~\%$ configurations while the straightforward semi-algebraic optimization fails in approx 34~\% of cases.

\noindent (4) We employ techniques from algebraic geometry~\cite{Cox-IVA-2015} and polynomial optimization~\cite{Lasserre2015} to solve the 7DOF IK problem exactly (within the numerical accuracy of computation). Our approach is also able to certify the in-feasibility of solving when it happens. 
%
\section{Previous work}
\noindent The first breakthrough in solving IK problems was the global solution to IK for a general 6DOF serial manipulator, which was given in~\cite{Raghaven:1991:KAM:112687.112715,DBLP:journals/trob/ManochaC94}. It leads to solving a polynomial system with 16 solutions. Another important result was the solution to the forward kinematics problem of the Stewart-Gough parallel manipulator platform~\cite{Lazard1993GeneralizedSP} leading to a polynomial system with 40 solutions. See recent work~\cite{Dai2019} for the review of local and other approximate techniques for solving IK problems. We next review only the most relevant work.
\subsection{The most relevant previous work}
The closest previous works are related to solving IK for mechanisms, which are under-constrained when considering positions of the final actuator only. The standard approach is to employ additional dynamics, time optimality, and collision constraints.  

In~\cite{Dai2014}, a technique for planning dynamic whole-body motion of a humanoid robot has been developed. It solves IK as a part of motion planning by local optimization methods taking into account kinematics, dynamics, and collision model. The planning method requires good initialization to converge, and depending on the quality of the initialization may take from minutes to hours of running time. Our approach provides a globally optimal solution for 7DOF kinematics subchains of more complex mechanisms and could be used to initialize kinematic part of motion planning. 

Work~\cite{Kuhlemann2016} presented IK solution for 7DOF manipulators with zero link offsets, e.g.\ KUKA LBR IIWA manipulators. The solution uses special kinematics of its class of manipulators to decompose the general IK problem to two simpler IK problems that can be solved in a closed-form. The one-dimensional variety of self-motions becomes circular, and hence the paper proposes to parameterize it by the angle of a point of the circle. Our approach generalizes this solution to a general 7DOF manipulator and shows that it is feasible to solve the IK problem for completely general 7DOF manipulators and optimize over their self-motion varieties. 

Paper~\cite{Dai2019} presents a global (approximate) solution to IK for 7DOF manipulators. It formulates the IK problem as a mixed-integer convex optimization program. The key idea of the paper is to approximate the non-convex space of rotations by piecewise linear functions on several intervals that partition the original space. This turns the original non-convex problem into an approximate convex problem when the right interval is chosen. Selecting the values of auxiliary binary variables to pick the actual interval of approximation leads to the integer part of the optimization. This is the first practical globally optimal approach, but it is only approximate and as such delivers solutions with errors in units of centimeters and units of degrees. It also fails to detect about $5~\%$ of infeasible poses. Our approach solves the original problem with sub-$10^{-4}$~mm and sub-$10^{-2}$~degree error and we can solve/decide the feasibility in all but $1~\%$ of tested cases. Computation times of~\cite{Dai2019} and our approach are roughly similar, in units of seconds. 
\section{Problem formulation}
\noindent Here we formulate the IK problem for the 7DOF serial manipulators as a semi-algebraic optimization problem with a polynomial objective function.

The task is to find joint coordinates of the manipulator in a way that the end-effector reaches the desired pose in space.
The IK problem is called under-constrained for manipulators, which have more DOF than they require to execute the given task.
In our case, to reach the desired pose in space, manipulators require to have six DOF, and therefore the IK problem for a 7DOF manipulator is under-constrained.
The consequence is that the IK problem has an infinite number of solutions for reachable generic end-effector poses for such manipulators.
This results in the self-motion property of these manipulators.
Self-motion is a motion of a manipulator, which is not observed in the task space, i.e.\ the end-effector pose of the manipulator is static while the links of the manipulator are moving.
Therefore, moving the manipulator along a path consisting of joint configurations of different solutions of the IK problem for the same pose in space will result in the self-motion of the manipulator.

The self-motion property provides the manipulator more adaptability since it allows, e.g.\ to avoid more obstacles in the paths and to avoid singularities, which leads to a more versatile mechanism.




On the other hand, increasing the degrees of freedom increases dramatically the difficulty of the IK problem computation.
The IK problem has no longer a finite number of solutions.
It can be formulated as a constrained optimization problem choosing the optimal solution from the set of all feasible solutions.
\subsection{Scope of the proposed method}
In this work, we present a general method for solving the IK problem for 7DOF serial manipulators.
We aim at a method that solves the IK problem, and that selects the globally optimal solution w.r.t.\ the given objective function from the infinite number of all feasible solutions.
It is naturally more time consuming to find the global solution than to find any solution, and therefore we do not expect our method to be an on-line method.
For on-line methods, such as used in the control units of the manipulators, the local methods are more suitable as they are fast and sufficiently accurate.

We see the application of our presented method in the developing process and the exploration of the capabilities of the manipulators.
The off-line method suits these tasks well as we are not typically limited by time.
Such a method can be used with an advantage when designing new 7DOF serial manipulators and optimizing their parameters, such as manipulability in regions of interest of the Cartesian space.
We see this as a reasonable approach as 7DOF serial manipulators are currently the most common redundant manipulators in the industry.

With regard to the presented scope, we next show how the IK problem for 7DOF serial manipulators can be modeled as a polynomial optimization problem (POP). 
\subsection{Description by forward kinematics}
\noindent We describe manipulators by Denavit-Hartenberg (D-H) convention~\cite{DH} to construct D-H transformation matrices $M_i(\theta_i)\in\R^{4\times4}$ from link $i$ to $i-1$. D-H matrices are parametrized by the joint angles $\theta_i$. The product of the D-H matrices for $i$ from $1$ to $7$ gives us the transformation matrix $M$, which represents the transformation from the end-effector coordinate system to the base coordinate system
\begin{align}
  \prod_{i=1}^7M_i(\theta_i) &= M.\labeleq{IKT:DKT}
\end{align}
The matrix $M$ consists of the position vector $t\in\R^3$ and the rotation matrix $R\in SO(3)$, which together represent the end-effector pose w.r.t.\ the base coordinate system.
When knowing the joint angles $\theta_i$, easy evaluation~\refeq{IKT:DKT} gives the end-effector pose in the base coordinate system.

Due to kinematics constraints, manipulators come with joint limits, i.e.\ with restrictions on the joint angles $\theta_i$. Typically, maximal $\theta_i^{High}$ and minimal $\theta_i^{Low}$ values of joint angles are given as
\begin{align}
  \theta_i^{Low} &\leq \theta_i \leq \theta_i^{High},\ i=1,\ldots,7.
\end{align}
\subsection{Inverse kinematics problem}\labelsec{IKT}
\noindent The forward kinematics problem is very easy to solve for serial manipulators.
On the other hand, the IK problem is much more difficult since it leads to solving systems of polynomial equations. To solve the IK problem we set up our desired pose of the end-effector in the form of matrix $M$ and then solve matrix \refeq{IKT:DKT} for the joint coordinates $\theta_i$.
For redundant manipulators, there is an infinite number of solutions, and therefore we introduce an objective function to select the solution on which the evaluation of the objective function is minimal.
In our case, we prefer the solutions that minimize the weighted sum of distances of the joint angles $\bm{\theta} = [\theta_1, \ldots, \theta_7]^\top$ from their preferred values $\bm{\hat{\theta}} = [\hat{\theta}_1, \ldots, \hat{\theta}_7]^\top$
\begin{align}
  \min_{\bm{\theta}\in\langle-\pi;\pi)^7} \sum_{i=1}^7 w_i\left((\theta_i - \hat{\theta}_i)\!\!\!\mod\pi\right),\labeleq{IKT:objectiveSOS}
\end{align}
where $w_i \geq 0$ and $\sum_{i=1}^7 w_i = 1$.
This objective function is widely used in the literature, e.g.\ \cite{pattacini2010experimental}.
In practice, the preferred values $\bm{\hat{\theta}}$ can be set to the previous configuration of the manipulator, and then the total movement of the actuators to reach the desired pose is minimized.

Next, we add joint limits to obtain the following optimization problem
\begin{align}
  \arraycolsep=1.4pt
  \begin{array}{lrcl@{\hskip0.5cm}l}
    \multicolumn{5}{l}{\displaystyle \min_{\bm{\theta}\in\langle-\pi;\pi)^7} \sum_{i=1}^7 w_i\left((\theta_i - \hat{\theta}_i)\!\!\!\mod\pi\right)} \\
    \text{s.t.} & \prod_{i=1}^7M_i(\theta_i) &=& M \\
    & \theta_i^{Low} \leq \theta_i &\leq& \theta_i^{High} & (i = 1,\ldots,7) \\
  \end{array}\labeleq{IKT:optimizationTask}
\end{align}

To be able to use techniques of polynomial optimization, we need to remove trigonometric functions that are contained in \refeq{IKT:DKT}. We do that by introducing new variables $\bm{c} = [c_1, \ldots, c_7]^\top$ and $\bm{s} = [s_1, \ldots, s_7]^\top$, which represent the cosines and sines of the joint angles $\bm{\theta} = [\theta_1, \ldots, \theta_7]^\top$ respectively. Then, we can rewrite Problem~\refeqb{IKT:optimizationTask} in the new variables. In order to preserve the structure, we need to add the trigonometric identities
\begin{align}
  q_i(\bm{c}, \bm{s}) &= c_i^2 + s_i^2 -1 = 0,\ i=1,\ldots,7. \labeleq{IKT:q}
\end{align}
Matrix \refeq{IKT:DKT} contains 12 trigonometric equations and can be directly rewritten as 12 polynomial equations of degrees up to seven in the newly introduced variables. However, we use the following clever manipulation with the matrix multiplication, which relies on the fact that the inverse of a rotation matrix is its transpose, i.e.\ it is a linear function of the original matrices,
\begin{align}
  \prod_{i=3}^5M_i(\theta_i) - M_2^{-1}(\theta_2)M_1^{-1}(\theta_1)MM_7^{-1}(\theta_7)M_6^{-1}(\theta_6) &= 0.\labeleq{IKT:DKT4}
\end{align}
It reduces the maximal degree of the polynomials in unknowns $\bm{c}$ and $\bm{s}$ to four.
We denote these polynomials in \refeq{IKT:DKT4} as
\begin{align}
  p_j(\bm{c}, \bm{s}) &= 0,\ j = 1,\ldots, 12 \labeleq{IKT:p}
\end{align}
The next step is to change objective~\refeqb{IKT:objectiveSOS} into a polynomial in the new variables $\bm{c}, \bm{s}$. We notice that the objective function~\refeqb{IKT:objectiveSOS} is minimal on the same solutions as the following objective function
\begin{align}
  & \min_{\bm{c}\in\langle-1,1\rangle^{7},\ \bm{s}\in\langle-1,1\rangle^7} \sum_{i=1}^7 w_i\Big((c_i - \cos\hat{\theta}_i)^2 + (s_i - \sin\hat{\theta}_i)^2\Big)\\
  &= \min_{\bm{c}\in\langle-1,1\rangle^{7},\ \bm{s}\in\langle-1,1\rangle^7} \sum_{i=1}^7 2w_i(1 - c_i\cos\hat{\theta}_i - s_i\sin\hat{\theta}_i)\labeleq{IKT:objectivePOP}.
\end{align}
After rewriting the joint limits inequalities into the polynomial form, we obtain the following polynomial optimization problem
\begin{align}
  \arraycolsep=1.4pt
  \begin{array}{lrcl@{\hskip0.5cm}l}
    \multicolumn{5}{l}{\displaystyle \min_{\bm{c}\in\langle-1,1\rangle^{7},\ \bm{s}\in\langle-1,1\rangle^7} \sum_{i=1}^7 2w_i(1 - c_i\cos\hat{\theta}_i - s_i\sin\hat{\theta}_i)} \\
    \text{s.t.} & p_j(\bm{c}, \bm{s}) &=& 0 & (j = 1,\ldots, 12) \\
    & q_i(\bm{c}, \bm{s}) &=& 0 & (i = 1,\ldots,7) \\
    & -(c_i+1)\tan\frac{\theta_i^{Low}}{2}+s_i &\geq&0 & (i = 1,\ldots,7) \\
    & (c_i+1)\tan\frac{\theta_i^{High}}{2}-s_i &\geq&0 & (i = 1,\ldots,7)
  \end{array}\labeleq{IKT:deg4}
\end{align}
We show how this polynomial optimization problem can be solved in the following sections.

Since the presented framework is general, any objective function can be chosen as long as it can be expressed as a low degree polynomial in sines and cosines of the joint angles.
Different objective functions will be chosen for different tasks, but we demonstrate the presented approach with the objective function \refeqb{IKT:objectivePOP}.

After solving Problem~\refeqb{IKT:deg4}, we recover $\bm{\theta}$ from $\bm{c}$ and $\bm{s}$ by function atan2, which takes into account signs of the arguments.
\section{Polynomial optimization}\labelsec{POP}
\noindent Next, we describe the polynomial optimization methods we use to solve Problem~\refeqb{IKT:deg4}.

Polynomial optimization problems (POPs) are generally non-convex, but they can be solved with global optimality certificates with the help of convex optimization, as surveyed in~\cite{Lasserre2015}. The idea consists of building a hierarchy of convex optimization problems of increasing size whose values converge to the value of the POP. The convergence proof is based on results of real algebraic geometry, namely the representation of positive polynomials, or Positivstellensatz (PSatz for short). One of the most popular Psatz is due to Putinar, and it expresses a polynomial positive on a compact basic semialgebraic set as a weighted sum of squares (SOS). Finding this SOS representation amounts to solving a semidefinite programming (SDP) problem, a particular convex optimization problem that can be solved efficiently numerically with interior point algorithms. By increasing the degree of the SOS representation, we increase the size of the SDP problem, thereby constructing a hierarchy of SDP problems. Dual to this polynomial positivity problem is the problem of characterizing moments of measures supported on a compact basic semialgebraic set. This also admits an SDP formulation, called moment relaxations, yielding a dual hierarchy, indexed by the so-called relaxation order. The primal-dual hierarchy is called the moment-SOS hierarchy or also the Lasserre hierarchy since it was first proposed in~\cite{Lasserre2001} in the context of POP with convergence and duality proofs. When the relaxation order increases, the Lasserre hierarchy generates a monotone sequence of superoptimal bounds on the global optimum of a given POP, and results on the moment problems can be used to certify exactness of a given bound, at a finite relaxation order. In this case, it is not necessary to go further in the hierarchy: the non-convex POP is solved at the price of solving a convex SDP problem of a given size. A Matlab package GloptiPoly \cite{Henrion2003} has been designed to construct the SDP problems in the hierarchy and solve them with a general-purpose SDP solver.

As observed in many applications, the main limitation of the Lasserre hierarchy (in its original form) is its poor scalability as a function of the number of variables and degree of the POP. This is balanced by the practical observation that, very often, global optimality is certified at the second or third-order relaxation. As our experiments reveal, for the degree 4 POP studied in our paper, the third order relaxation is out of reach of state-of-the-art SDP solvers. It becomes hence critical to investigate reformulation techniques to reduce the degree as much as possible. This is the topic of the next section.
\section{Symbolic reduction of the POP}
\noindent Here we provide the description of the algebraic geometry technique we use to reduce the degree of our POP problem to obtain a practical solving method. See~\cite{Cox-IVA-2015} for algebraic-geometric notation and concepts.

The POP we have at hand is a constrained with polynomial equations
\begin{align}
f_1=\cdots=f_s=0
\end{align}
of degree $4$ in $\mathbb{Q}[x_1, \ldots, x_n]$. Observe that one can replace
these polynomial equations in the formulation of the POP by any other set of
polynomial equations
\begin{align}
g_1=\cdots=g_t=0
\end{align}
as long as both systems of equations have the same solution set. Natural
candidates for the $g_i$'s are to pick them in the {\em ideal} generated by
$(f_1, \ldots, f_s)$, i.e. the set of algebraic combinations $I = \{\sum_i q_i
f_i\mid q_i \in \mathbb{Q}[x_1, \ldots, x_n]\}$. It is clear that if all the
$f_i$'s vanish simultaneously at a point, any polynomial $g$ in this set will
vanish at this point.

The difficulty is how to understand the structure of this set and find a nice
finite representation of it that would allow many algebraic operations (such as
deciding whether a given polynomial lies in this set). Solutions have been
brought by symbolic computation, aka computer algebra, through the development
of algorithms computing {\em Gr\"obner bases}, which were introduced by
Buchberger~\cite{Cox-IVA-2015}. These are finite sets, depending on a monomial ordering~\cite{Cox-IVA-2015}, which generate $I$ as input equations do, but from which
the whole structure of $I$ can be read.

Modern algorithms for computing Gr\"obner bases ($F4$ and $F5$ algorithms),
which significantly improved by several orders of magnitude the
state-of-the-art, were introduced next by J.~C.~Faug\`ere~\cite{F4, F5}. These
latter algorithms bring a linear algebra approach to Gr\"obner bases
computations. In particular, noticing that the intersection of $I$ with the
subset of polynomials in $\mathbb{Q}[x_1, \ldots, x_n]$ of degree $\leq d$ is a
vector space of finite dimension, is a key to reduce Gr\"obner bases
computations to exact linear algebra operations.

Hence, Gr\"obner bases provide bases of such vector spaces when one uses
monomial orderings which filter monomials w.r.t.\ degree first. Finally, going
back to our problem, a Gr\"obner basis computation allows us to discover if $I$
contains degree $2$ polynomials (and is generated by such quadrics).

While this is never the case when starting with generic degree $4$, observe that
there are many relations between the coefficients of the degree $4$ equations of
our POP. Hence, we are not facing a generic situation there and we'll see
further that actually a Gr\"obner basis computation provides a set of quadrics
that can replace our initial set of constraints. Note also that since Gr\"obner
basis algorithms rely on exact linear algebra, such a property holds for any
instance of our POP if it holds for a randomly chosen one (the trace of the
computation will always be the same, giving rise to polynomials of degree $\leq
2$).
\section{Solving the IK problem}
\noindent In order to solve the IK problem, we need to solve the optimization problem \refeqb{IKT:deg4}. First, we apply the implementation GloptiPoly \cite{Gloptipoly} of the method described in \refsec{POP} directly on the Problem \refeqb{IKT:deg4}.
\subsection{Direct application of polynomial solver}
\noindent Since the original Problem \refeqb{IKT:deg4} contains polynomials of degrees up to 4, we start with the first relaxation of order two. That means we substitute each monomial in the original 14 variables up to degree four by a new variable, and therefore the resulting SDP program will have \num{3060} variables.

Solving the first relaxation typically does not yield the solution for this parametrization of the problem, and therefore it is required to go higher in the relaxation hierarchy. Unfortunately, relaxation order three for a polynomial problem in 14 variables leads to an SDP problem in \num{38760} variables. Such a huge problem is still often solvable on contemporary computers, but it often takes hours to finish.
\subsection{Symbolic reduction}\labelsec{sym}
\noindent In the view of the previous paragraph, we aim at simplifying the original polynomial problem to be able to obtain solutions even for the relaxation of order two, which takes seconds to solve.

Here is our main result that allows us to do it. We claim that polynomials $p_j$ and $q_i$ of degrees up to four in Problem \refeqb{IKT:deg4} can be reduced to polynomials of degree two.
\begin{theorem}\label{THM}
  The ideal generated by the kinematics constraints \refeqb{IKT:p} for a generic serial manipulator with seven revolute joints and for generic pose $M$ with the addition of the trigonometric identities \refeqb{IKT:q} can be generated by a set of degree two polynomials.
\end{theorem}
\begin{proof}
  The proof is computational.
  We generate generic instances of serial manipulators and generic poses.
  Then a Gr\"obner basis $G$ \cite{Cox1996} of polynomials $p_j$ and $q_i$ is computed for each instance of the manipulator and pose.
  We select a subset $S$ of degree two polynomials from the basis $G$, and by computing a new Gr\"obner basis $G^\prime$ from $S$, we verify that $S$ generates the same ideal as the original set of polynomials.
  See Maple code in~\reflis{maple}.
  The polynomials $p_j$ and $q_i$ are put into the variable \texttt{eq}, and the last command of the code will be evaluated to \texttt{True} if the bases $G$ and $G^\prime$ are equal, and therefore generate the same ideal.
\begin{lstlisting}[language=maple, caption={Maple code for the proof of Theorem~\ref{THM}.}, labellis={maple}]
# compute the reduced Groebner basis from pj and qi polynomials (in variables of eq)
G := Basis(eq, tdeg(op(indets(eq)))):
# select degree two polynomials from the basis and compute a new reduced Groebner basis
idxDegTwo := SearchAll(2, map(degree, G)):
eqPrime := G[[idxDegTwo]]:
GPrime := Basis(eqPrime, tdeg(op(indets(eq)))):
# compare the two bases
evalb(G = GPrime);
@@\textit{\color{blue}{\hfill True\hfill}}
\end{lstlisting}
\end{proof}
\newcommand{\GBFailed}{1.2 \%}
\newcommand{\GBZeroErrorsT}{0}
\newcommand{\GBZeroErrorsR}{0}
\newcommand{\GBErrorMeanT}{\num{7.27E-05}}
\newcommand{\GBErrorMeanR}{\num{5.59E-03}}
\newcommand{\GBTimeMeanGloptipoly}{\num{ 5.6}}
\newcommand{\GBTimeMeanGB}{\num{ 2.7}}
\newcommand{\gloptipolyFailed}{32.4 \%}
\newcommand{\gloptipolyZeroErrorsT}{0}
\newcommand{\gloptipolyZeroErrorsR}{509}
\newcommand{\gloptipolyErrorMeanT}{\num{3.92E-04}}
\newcommand{\gloptipolyErrorMeanR}{\num{6.11E-05}}
\newcommand{\gloptipolyTimeMeanGloptipoly}{\num{21.3}}
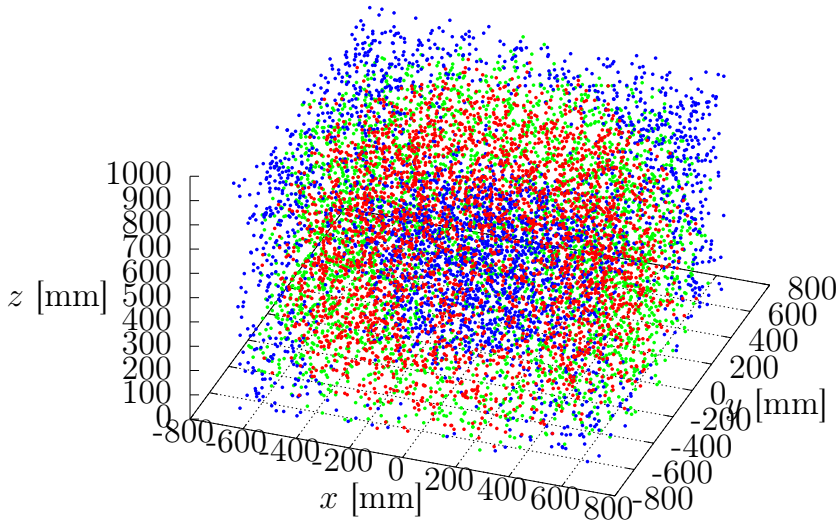
\begin{figure}[t]
  \centering
  \resizebox{0.8\textwidth}{!}{
\begingroup
  \makeatletter
  \gdef\gplbacktext{}%
  \gdef\gplfronttext{}%
  \makeatother
  \ifGPblacktext
    \def\colorrgb#1{}%
    \def\colorgray#1{}%
  \else
    \ifGPcolor
      \def\colorrgb#1{\color[rgb]{#1}}%
      \def\colorgray#1{\color[gray]{#1}}%
      \expandafter\def\csname LTw\endcsname{\color{white}}%
      \expandafter\def\csname LTb\endcsname{\color{black}}%
      \expandafter\def\csname LTa\endcsname{\color{black}}%
      \expandafter\def\csname LT0\endcsname{\color[rgb]{1,0,0}}%
      \expandafter\def\csname LT1\endcsname{\color[rgb]{0,1,0}}%
      \expandafter\def\csname LT2\endcsname{\color[rgb]{0,0,1}}%
      \expandafter\def\csname LT3\endcsname{\color[rgb]{1,0,1}}%
      \expandafter\def\csname LT4\endcsname{\color[rgb]{0,1,1}}%
      \expandafter\def\csname LT5\endcsname{\color[rgb]{1,1,0}}%
      \expandafter\def\csname LT6\endcsname{\color[rgb]{0,0,0}}%
      \expandafter\def\csname LT7\endcsname{\color[rgb]{1,0.3,0}}%
      \expandafter\def\csname LT8\endcsname{\color[rgb]{0.5,0.5,0.5}}%
    \else
      \def\colorrgb#1{\color{black}}%
      \def\colorgray#1{\color[gray]{#1}}%
      \expandafter\def\csname LTw\endcsname{\color{white}}%
      \expandafter\def\csname LTb\endcsname{\color{black}}%
      \expandafter\def\csname LTa\endcsname{\color{black}}%
      \expandafter\def\csname LT0\endcsname{\color{black}}%
      \expandafter\def\csname LT1\endcsname{\color{black}}%
      \expandafter\def\csname LT2\endcsname{\color{black}}%
      \expandafter\def\csname LT3\endcsname{\color{black}}%
      \expandafter\def\csname LT4\endcsname{\color{black}}%
      \expandafter\def\csname LT5\endcsname{\color{black}}%
      \expandafter\def\csname LT6\endcsname{\color{black}}%
      \expandafter\def\csname LT7\endcsname{\color{black}}%
      \expandafter\def\csname LT8\endcsname{\color{black}}%
    \fi
  \fi
    \setlength{\unitlength}{0.0500bp}%
    \ifx\gptboxheight\undefined%
      \newlength{\gptboxheight}%
      \newlength{\gptboxwidth}%
      \newsavebox{\gptboxtext}%
    \fi%
    \setlength{\fboxrule}{0.5pt}%
    \setlength{\fboxsep}{1pt}%
\begin{picture}(5760.00,4320.00)%
    \gplgaddtomacro\gplbacktext{%
      \csname LTb\endcsname
      \put(949,726){\makebox(0,0){\strut{}-800}}%
      \csname LTb\endcsname
      \put(1321,659){\makebox(0,0){\strut{}-600}}%
      \csname LTb\endcsname
      \put(1693,591){\makebox(0,0){\strut{}-400}}%
      \csname LTb\endcsname
      \put(2065,523){\makebox(0,0){\strut{}-200}}%
      \csname LTb\endcsname
      \put(2438,456){\makebox(0,0){\strut{}0}}%
      \csname LTb\endcsname
      \put(2810,388){\makebox(0,0){\strut{}200}}%
      \csname LTb\endcsname
      \put(3181,320){\makebox(0,0){\strut{}400}}%
      \csname LTb\endcsname
      \put(3553,252){\makebox(0,0){\strut{}600}}%
      \csname LTb\endcsname
      \put(3925,185){\makebox(0,0){\strut{}800}}%
      \csname LTb\endcsname
      \put(4094,250){\makebox(0,0)[l]{\strut{}-800}}%
      \csname LTb\endcsname
      \put(4230,436){\makebox(0,0)[l]{\strut{}-600}}%
      \csname LTb\endcsname
      \put(4365,623){\makebox(0,0)[l]{\strut{}-400}}%
      \csname LTb\endcsname
      \put(4501,809){\makebox(0,0)[l]{\strut{}-200}}%
      \csname LTb\endcsname
      \put(4636,995){\makebox(0,0)[l]{\strut{}0}}%
      \csname LTb\endcsname
      \put(4771,1181){\makebox(0,0)[l]{\strut{}200}}%
      \csname LTb\endcsname
      \put(4907,1367){\makebox(0,0)[l]{\strut{}400}}%
      \csname LTb\endcsname
      \put(5042,1553){\makebox(0,0)[l]{\strut{}600}}%
      \csname LTb\endcsname
      \put(5178,1739){\makebox(0,0)[l]{\strut{}800}}%
      \put(868,830){\makebox(0,0)[r]{\strut{}0}}%
      \put(868,1001){\makebox(0,0)[r]{\strut{}100}}%
      \put(868,1173){\makebox(0,0)[r]{\strut{}200}}%
      \put(868,1344){\makebox(0,0)[r]{\strut{}300}}%
      \put(868,1516){\makebox(0,0)[r]{\strut{}400}}%
      \put(868,1687){\makebox(0,0)[r]{\strut{}500}}%
      \put(868,1859){\makebox(0,0)[r]{\strut{}600}}%
      \put(868,2030){\makebox(0,0)[r]{\strut{}700}}%
      \put(868,2201){\makebox(0,0)[r]{\strut{}800}}%
      \put(868,2372){\makebox(0,0)[r]{\strut{}900}}%
      \put(868,2544){\makebox(0,0)[r]{\strut{}1000}}%
      \put(3024,3876){\makebox(0,0){\strut{}}}%
    }%
    \gplgaddtomacro\gplfronttext{%
      \csname LTb\endcsname
      \put(2267,263){\makebox(0,0){\strut{}$x$ [mm]}}%
      \put(5104,925){\makebox(0,0){\strut{}$y$ [mm]}}%
      \put(70,1687){\makebox(0,0){\strut{}$z$ [mm]}}%
    }%
    \gplbacktext
    \put(0,0){\includegraphics{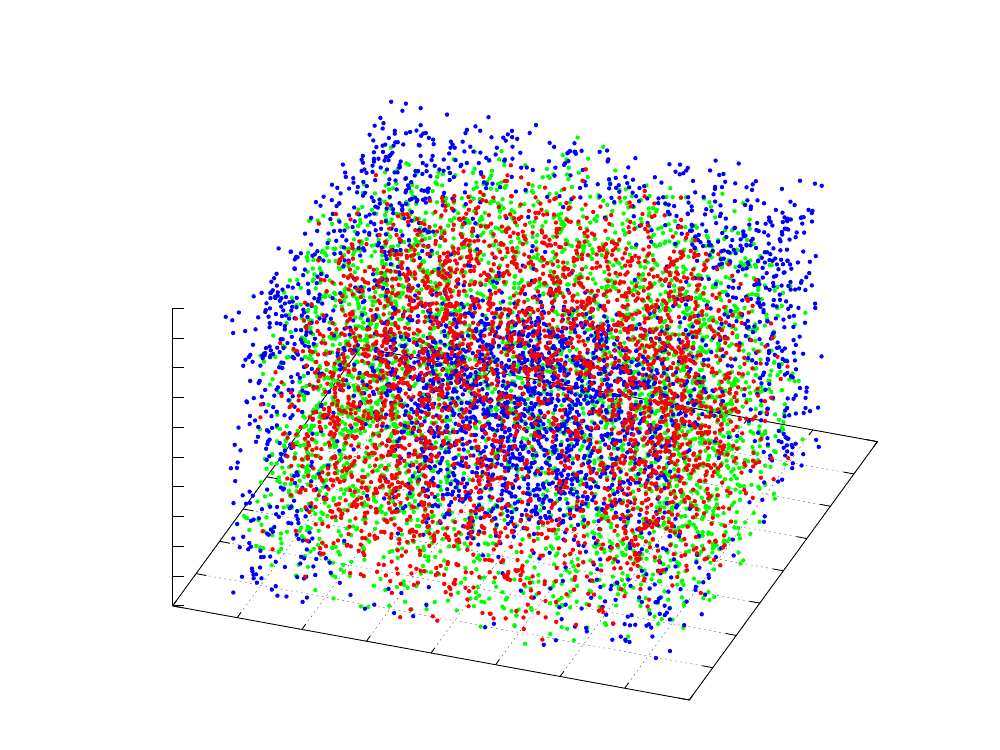}}%
    \gplfronttext
  \end{picture}%
\endgroup
  }
  \caption{Generated poses of the manipulator. Green dots are poses marked as feasible by direct solving with GloptiPoly, blue as infeasible, and for the red ones the computation failed (\gloptipolyFailed).}
  \labelfig{Ill:gloptipoly:3D}
\end{figure}
\subsection{Solving the reduced polynomial optimization problem}
\noindent We exploit Theorem \ref{THM} in our approach to solve the IK problem. First, we compute a Gr\"obner basis of the kinematics constraints \refeqb{IKT:p} and \refeqb{IKT:q} from which we select only polynomials of degree two. Then, we construct the Problem \refeqb{IKT:deg4} but with polynomial constraints given by the degree two polynomials only.
We solve the problem by hierarchies of semidefinite programs.

Reducing the degree of polynomials from four to two allows us to start with SDP relaxation of order one.
The size of this SDP problem, in terms of the number of variables, is now 120.
Practical experiments have shown that the first relaxation is not tight enough to yield the solution.
On the other hand, the second relaxation gives a solution for almost all poses, see \reftab{ill:results}.
\begin{figure}[t]
  \centering
  \resizebox{0.8\textwidth}{!}{
\begingroup
  \makeatletter
  \gdef\gplbacktext{}%
  \gdef\gplfronttext{}%
  \makeatother
  \ifGPblacktext
    \def\colorrgb#1{}%
    \def\colorgray#1{}%
  \else
    \ifGPcolor
      \def\colorrgb#1{\color[rgb]{#1}}%
      \def\colorgray#1{\color[gray]{#1}}%
      \expandafter\def\csname LTw\endcsname{\color{white}}%
      \expandafter\def\csname LTb\endcsname{\color{black}}%
      \expandafter\def\csname LTa\endcsname{\color{black}}%
      \expandafter\def\csname LT0\endcsname{\color[rgb]{1,0,0}}%
      \expandafter\def\csname LT1\endcsname{\color[rgb]{0,1,0}}%
      \expandafter\def\csname LT2\endcsname{\color[rgb]{0,0,1}}%
      \expandafter\def\csname LT3\endcsname{\color[rgb]{1,0,1}}%
      \expandafter\def\csname LT4\endcsname{\color[rgb]{0,1,1}}%
      \expandafter\def\csname LT5\endcsname{\color[rgb]{1,1,0}}%
      \expandafter\def\csname LT6\endcsname{\color[rgb]{0,0,0}}%
      \expandafter\def\csname LT7\endcsname{\color[rgb]{1,0.3,0}}%
      \expandafter\def\csname LT8\endcsname{\color[rgb]{0.5,0.5,0.5}}%
    \else
      \def\colorrgb#1{\color{black}}%
      \def\colorgray#1{\color[gray]{#1}}%
      \expandafter\def\csname LTw\endcsname{\color{white}}%
      \expandafter\def\csname LTb\endcsname{\color{black}}%
      \expandafter\def\csname LTa\endcsname{\color{black}}%
      \expandafter\def\csname LT0\endcsname{\color{black}}%
      \expandafter\def\csname LT1\endcsname{\color{black}}%
      \expandafter\def\csname LT2\endcsname{\color{black}}%
      \expandafter\def\csname LT3\endcsname{\color{black}}%
      \expandafter\def\csname LT4\endcsname{\color{black}}%
      \expandafter\def\csname LT5\endcsname{\color{black}}%
      \expandafter\def\csname LT6\endcsname{\color{black}}%
      \expandafter\def\csname LT7\endcsname{\color{black}}%
      \expandafter\def\csname LT8\endcsname{\color{black}}%
    \fi
  \fi
    \setlength{\unitlength}{0.0500bp}%
    \ifx\gptboxheight\undefined%
      \newlength{\gptboxheight}%
      \newlength{\gptboxwidth}%
      \newsavebox{\gptboxtext}%
    \fi%
    \setlength{\fboxrule}{0.5pt}%
    \setlength{\fboxsep}{1pt}%
\begin{picture}(5760.00,4320.00)%
    \gplgaddtomacro\gplbacktext{%
      \csname LTb\endcsname
      \put(949,726){\makebox(0,0){\strut{}-800}}%
      \csname LTb\endcsname
      \put(1321,659){\makebox(0,0){\strut{}-600}}%
      \csname LTb\endcsname
      \put(1693,591){\makebox(0,0){\strut{}-400}}%
      \csname LTb\endcsname
      \put(2065,523){\makebox(0,0){\strut{}-200}}%
      \csname LTb\endcsname
      \put(2438,456){\makebox(0,0){\strut{}0}}%
      \csname LTb\endcsname
      \put(2810,388){\makebox(0,0){\strut{}200}}%
      \csname LTb\endcsname
      \put(3181,320){\makebox(0,0){\strut{}400}}%
      \csname LTb\endcsname
      \put(3553,252){\makebox(0,0){\strut{}600}}%
      \csname LTb\endcsname
      \put(3925,185){\makebox(0,0){\strut{}800}}%
      \csname LTb\endcsname
      \put(4094,250){\makebox(0,0)[l]{\strut{}-800}}%
      \csname LTb\endcsname
      \put(4230,436){\makebox(0,0)[l]{\strut{}-600}}%
      \csname LTb\endcsname
      \put(4365,623){\makebox(0,0)[l]{\strut{}-400}}%
      \csname LTb\endcsname
      \put(4501,809){\makebox(0,0)[l]{\strut{}-200}}%
      \csname LTb\endcsname
      \put(4636,995){\makebox(0,0)[l]{\strut{}0}}%
      \csname LTb\endcsname
      \put(4771,1181){\makebox(0,0)[l]{\strut{}200}}%
      \csname LTb\endcsname
      \put(4907,1367){\makebox(0,0)[l]{\strut{}400}}%
      \csname LTb\endcsname
      \put(5042,1553){\makebox(0,0)[l]{\strut{}600}}%
      \csname LTb\endcsname
      \put(5178,1739){\makebox(0,0)[l]{\strut{}800}}%
      \put(868,830){\makebox(0,0)[r]{\strut{}0}}%
      \put(868,1001){\makebox(0,0)[r]{\strut{}100}}%
      \put(868,1173){\makebox(0,0)[r]{\strut{}200}}%
      \put(868,1344){\makebox(0,0)[r]{\strut{}300}}%
      \put(868,1516){\makebox(0,0)[r]{\strut{}400}}%
      \put(868,1687){\makebox(0,0)[r]{\strut{}500}}%
      \put(868,1859){\makebox(0,0)[r]{\strut{}600}}%
      \put(868,2030){\makebox(0,0)[r]{\strut{}700}}%
      \put(868,2201){\makebox(0,0)[r]{\strut{}800}}%
      \put(868,2372){\makebox(0,0)[r]{\strut{}900}}%
      \put(868,2544){\makebox(0,0)[r]{\strut{}1000}}%
      \put(3024,3876){\makebox(0,0){\strut{}}}%
    }%
    \gplgaddtomacro\gplfronttext{%
      \csname LTb\endcsname
      \put(2267,263){\makebox(0,0){\strut{}$x$ [mm]}}%
      \put(5104,925){\makebox(0,0){\strut{}$y$ [mm]}}%
      \put(70,1687){\makebox(0,0){\strut{}$z$ [mm]}}%
    }%
    \gplbacktext
    \put(0,0){\includegraphics{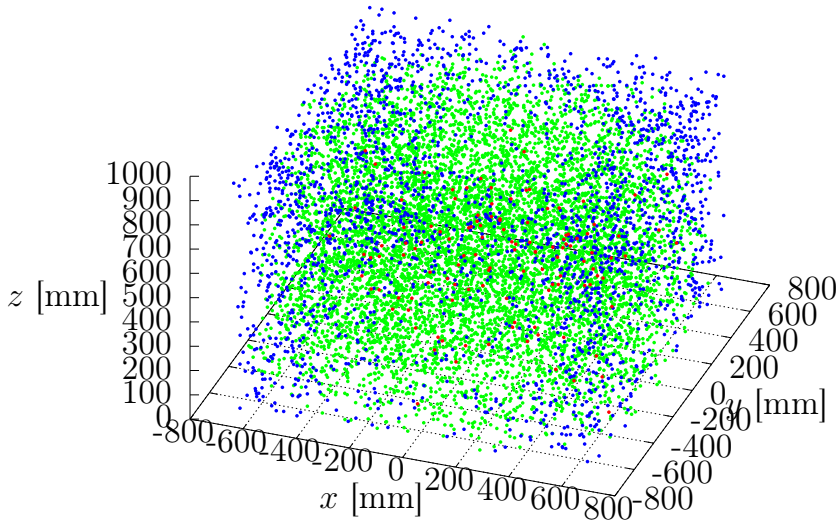}}%
    \gplfronttext
  \end{picture}%
\endgroup
  }
  \caption{Generated poses of the manipulator. Green dots are poses marked as feasible by GloptiPoly after symbolic simplification, blue as infeasible, and for the red ones the computation failed (\GBFailed).}
  \labelfig{Ill:GB:3D}
\end{figure}
\section{Experiments}
\begin{table*}
  \centering
  \caption{Overview of execution times and accuracy of the presented methods.}
  \labeltab{ill:results}
  \resizebox{\textwidth}{!}{
  \begin{tabular}{cccccc}\hline
    & \multicolumn{2}{c}{\textbf{Execution time} [s]} & \multicolumn{2}{c}{\textbf{Median error}} & \textbf{\% of failed}\\
    & \textbf{Reduction step} & \textbf{GloptiPoly} & \textbf{Tran.} [mm] & \textbf{Rot.} [deg] & \textbf{poses}\\\hline
    Deg. 4 pol. & ---           & \gloptipolyTimeMeanGloptipoly& \gloptipolyErrorMeanT & \gloptipolyErrorMeanR & \gloptipolyFailed \\
    Deg. 2 pol. & \GBTimeMeanGB & \GBTimeMeanGloptipoly        & \GBErrorMeanT         & \GBErrorMeanR         & \GBFailed         \\\hline
  \end{tabular}
  }
\end{table*}

\noindent We demonstrate our method on IK problem for \textit{KUKA LBR IIWA} arm with seven revolute joints.
The structure of the manipulator is designed in a special way such that the IK problem is simple to compute.
One of the advantages is that for a fixed end-effector pose, the joint angle $\theta_4$ is constant within the self-motion.
This allows for a geometrical derivation of a closed-form solution to the IK problem, such a \cite{Kuhlemann2016}, where the authors introduce new angle parameter $\delta$ that fixes the left DOF of the IK problem.

Another approach is to solve the problem by local non-linear optimization techniques \cite{Buss2004}, but such methods do not provide global optima, and the found solution is highly dependent on the initial guess.

Solving the IK problem globally is more computationally challenging.
To be able to tackle the problem in a matter of seconds, relaxations of the problem were developed in the past.
Dai et al.\ in \cite{Dai2019} proposed mix-integer convex relaxation of the non-convex rotational constraints.
Their method finds all classes of solutions that are in correspondence with a different set of active binary variables, but they are unable to select a global optima w.r.t.\ an objective function.
\subsection{Polynomial optimization problem for KUKA LBR IIWA}
\noindent We directly parameterize Problem \refeqb{IKT:deg4} by the D-H parameters of the \textit{KUKA LBR IIWA} manipulator.
We set the weights equally to $w_i = \frac{1}{7}$, and we set the preferred values of $\hat{\theta}_i$ to zero, which is in the middle of the joint allowed interval.
This leads to POP in 14 variables and with polynomials $p_j$ of degrees up to four.
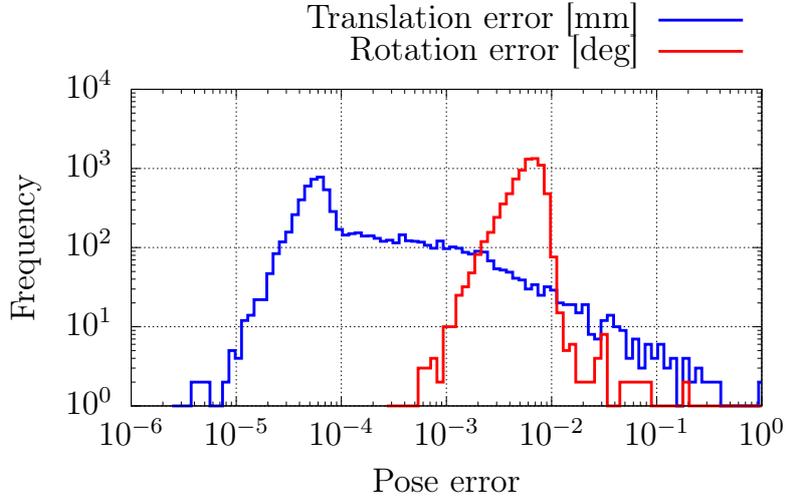
\begin{figure}[t]
  \centering
  \resizebox{0.8\textwidth}{!}{
\begingroup
  \makeatletter
  \gdef\gplbacktext{}%
  \gdef\gplfronttext{}%
  \makeatother
  \ifGPblacktext
    \def\colorrgb#1{}%
    \def\colorgray#1{}%
  \else
    \ifGPcolor
      \def\colorrgb#1{\color[rgb]{#1}}%
      \def\colorgray#1{\color[gray]{#1}}%
      \expandafter\def\csname LTw\endcsname{\color{white}}%
      \expandafter\def\csname LTb\endcsname{\color{black}}%
      \expandafter\def\csname LTa\endcsname{\color{black}}%
      \expandafter\def\csname LT0\endcsname{\color[rgb]{1,0,0}}%
      \expandafter\def\csname LT1\endcsname{\color[rgb]{0,1,0}}%
      \expandafter\def\csname LT2\endcsname{\color[rgb]{0,0,1}}%
      \expandafter\def\csname LT3\endcsname{\color[rgb]{1,0,1}}%
      \expandafter\def\csname LT4\endcsname{\color[rgb]{0,1,1}}%
      \expandafter\def\csname LT5\endcsname{\color[rgb]{1,1,0}}%
      \expandafter\def\csname LT6\endcsname{\color[rgb]{0,0,0}}%
      \expandafter\def\csname LT7\endcsname{\color[rgb]{1,0.3,0}}%
      \expandafter\def\csname LT8\endcsname{\color[rgb]{0.5,0.5,0.5}}%
    \else
      \def\colorrgb#1{\color{black}}%
      \def\colorgray#1{\color[gray]{#1}}%
      \expandafter\def\csname LTw\endcsname{\color{white}}%
      \expandafter\def\csname LTb\endcsname{\color{black}}%
      \expandafter\def\csname LTa\endcsname{\color{black}}%
      \expandafter\def\csname LT0\endcsname{\color{black}}%
      \expandafter\def\csname LT1\endcsname{\color{black}}%
      \expandafter\def\csname LT2\endcsname{\color{black}}%
      \expandafter\def\csname LT3\endcsname{\color{black}}%
      \expandafter\def\csname LT4\endcsname{\color{black}}%
      \expandafter\def\csname LT5\endcsname{\color{black}}%
      \expandafter\def\csname LT6\endcsname{\color{black}}%
      \expandafter\def\csname LT7\endcsname{\color{black}}%
      \expandafter\def\csname LT8\endcsname{\color{black}}%
    \fi
  \fi
    \setlength{\unitlength}{0.0500bp}%
    \ifx\gptboxheight\undefined%
      \newlength{\gptboxheight}%
      \newlength{\gptboxwidth}%
      \newsavebox{\gptboxtext}%
    \fi%
    \setlength{\fboxrule}{0.5pt}%
    \setlength{\fboxsep}{1pt}%
\begin{picture}(5760.00,3600.00)%
    \gplgaddtomacro\gplbacktext{%
      \csname LTb\endcsname
      \put(814,704){\makebox(0,0)[r]{\strut{}$10^{0}$}}%
      \csname LTb\endcsname
      \put(814,1263){\makebox(0,0)[r]{\strut{}$10^{1}$}}%
      \csname LTb\endcsname
      \put(814,1822){\makebox(0,0)[r]{\strut{}$10^{2}$}}%
      \csname LTb\endcsname
      \put(814,2380){\makebox(0,0)[r]{\strut{}$10^{3}$}}%
      \csname LTb\endcsname
      \put(814,2939){\makebox(0,0)[r]{\strut{}$10^{4}$}}%
      \csname LTb\endcsname
      \put(946,484){\makebox(0,0){\strut{}$10^{-6}$}}%
      \csname LTb\endcsname
      \put(1682,484){\makebox(0,0){\strut{}$10^{-5}$}}%
      \csname LTb\endcsname
      \put(2418,484){\makebox(0,0){\strut{}$10^{-4}$}}%
      \csname LTb\endcsname
      \put(3154,484){\makebox(0,0){\strut{}$10^{-3}$}}%
      \csname LTb\endcsname
      \put(3891,484){\makebox(0,0){\strut{}$10^{-2}$}}%
      \csname LTb\endcsname
      \put(4627,484){\makebox(0,0){\strut{}$10^{-1}$}}%
      \csname LTb\endcsname
      \put(5363,484){\makebox(0,0){\strut{}$10^{0}$}}%
    }%
    \gplgaddtomacro\gplfronttext{%
      \csname LTb\endcsname
      \put(209,1821){\rotatebox{-270}{\makebox(0,0){\strut{}Frequency}}}%
      \put(3154,154){\makebox(0,0){\strut{}Pose error}}%
      \put(3154,2829){\makebox(0,0){\strut{}}}%
      \csname LTb\endcsname
      \put(4508,3427){\makebox(0,0)[r]{\strut{}Translation error [mm]}}%
      \csname LTb\endcsname
      \put(4508,3207){\makebox(0,0)[r]{\strut{}Rotation error [deg]}}%
    }%
    \gplbacktext
    \put(0,0){\includegraphics{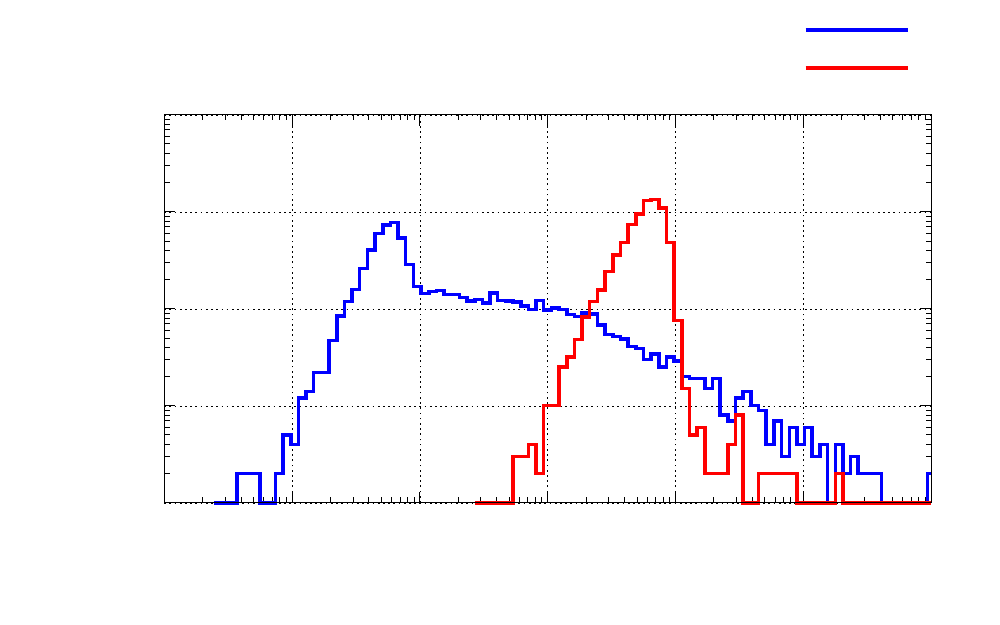}}%
    \gplfronttext
  \end{picture}%
\endgroup
  }
  \caption{Histogram of translation and rotation error of the poses computed from the forward kinematics on found solutions w.r.t.\ the desired poses. There are \GBZeroErrorsT{} zero translation errors and \GBZeroErrorsR{} zero rotation errors.}
  \labelfig{Ill:GB:errors}
\end{figure}
\subsection{Direct application of polynomial solver}\labelsec{ill:naive}
\noindent First, we solve Problem \refeqb{IKT:deg4} directly by polynomial optimization toolbox GloptiPoly \cite{Gloptipoly} for relaxation order two with the use of MOSEK \cite{Mosek} as the semidefinite problem solver.

Our dataset consists of \num{10000} randomly chosen poses within and outside of the working space of the manipulator, as shown in~\reffig{Ill:gloptipoly:3D}. For poses marked by red color, GloptiPoly failed to compute the solution or report infeasibility. That is mainly due to the small relaxation order of the semidefinite relaxation of the POP. There is \gloptipolyFailed{} of such poses, which makes this approach quite impractical. Computations for the next degree three relaxation is still often feasible on contemporary computers but takes hours to finish.
\subsection{POP with symbolic reduction}
\noindent Since the performance of GloptiPoly highly depends on the number of variables of the POP and the relaxation degree, which grows with the degrees of the polynomials contained in the POP, we first symbolically reduce polynomials $p_j$ and $q_i$ and then solve the resulting POP by GloptiPoly.

Firstly, we take advantage of the simplified structure of the \textit{KUKA LBR IIWA} manipulator, i.e. that the joint angle $\theta_4$ is constant within the self-motion, and therefore it plays no role in the objective function \refeqb{IKT:objectiveSOS}. That allows us to eliminate the variables $c_4$ and $s_4$ from the equations. Secondly, we reduce the polynomials $p_j$ and $q_j$ symbolically with the use of Theorem~\ref{THM}.

In this way, we have reduced the number of variables from 14 to 12, and we have reduced the degrees of the polynomials to two, which significantly speeds up the SDP solver. Practical experiments showed that GloptiPoly is now able to compute IK for more poses with the same relaxation order two than by the na\"ive approach used before, see \reffig{Ill:GB:3D}. Now only \GBFailed{} of poses failed to be solved on the same dataset as in \refsec{ill:naive}.

To verify the numerical stability of the solver, we have computed the forward kinematics problem based on the found joint angles from the IK problem. Then, we have computed the translation error and rotation error of this pose w.r.t.\ the desired pose. The histogram of the translation and rotation error can be seen in \reffig{Ill:GB:errors}.

For practical applications, the execution time of this method is important. In \reffig{Ill:GB:times}, we show histograms of the execution time of the on-line phase of GloptiPoly as well as of the symbolic reduction of the initial polynomials to degree two polynomials. We observe that our execution times are comparable to computation times in~\cite{Dai2019} when using off-the-shelf POP and GB computation tools. We next plan to develop optimized solvers leading to considerable speedup, as it was done in solving polynomial systems in computer vision~\cite{Larsson-CVPR-2018}.

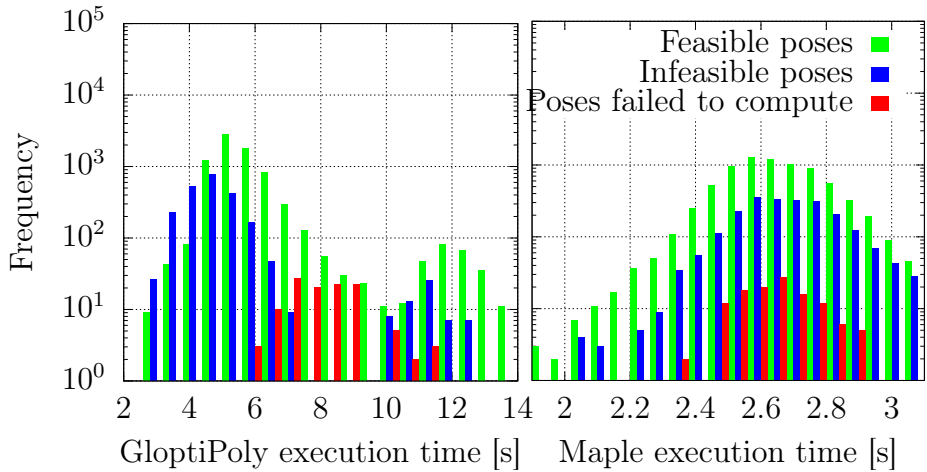
\begin{figure}[t]
  \resizebox{0.5125\textwidth}{!}{
\begingroup
  \makeatletter
  \gdef\gplbacktext{}%
  \gdef\gplfronttext{}%
  \makeatother
  \ifGPblacktext
    \def\colorrgb#1{}%
    \def\colorgray#1{}%
  \else
    \ifGPcolor
      \def\colorrgb#1{\color[rgb]{#1}}%
      \def\colorgray#1{\color[gray]{#1}}%
      \expandafter\def\csname LTw\endcsname{\color{white}}%
      \expandafter\def\csname LTb\endcsname{\color{black}}%
      \expandafter\def\csname LTa\endcsname{\color{black}}%
      \expandafter\def\csname LT0\endcsname{\color[rgb]{1,0,0}}%
      \expandafter\def\csname LT1\endcsname{\color[rgb]{0,1,0}}%
      \expandafter\def\csname LT2\endcsname{\color[rgb]{0,0,1}}%
      \expandafter\def\csname LT3\endcsname{\color[rgb]{1,0,1}}%
      \expandafter\def\csname LT4\endcsname{\color[rgb]{0,1,1}}%
      \expandafter\def\csname LT5\endcsname{\color[rgb]{1,1,0}}%
      \expandafter\def\csname LT6\endcsname{\color[rgb]{0,0,0}}%
      \expandafter\def\csname LT7\endcsname{\color[rgb]{1,0.3,0}}%
      \expandafter\def\csname LT8\endcsname{\color[rgb]{0.5,0.5,0.5}}%
    \else
      \def\colorrgb#1{\color{black}}%
      \def\colorgray#1{\color[gray]{#1}}%
      \expandafter\def\csname LTw\endcsname{\color{white}}%
      \expandafter\def\csname LTb\endcsname{\color{black}}%
      \expandafter\def\csname LTa\endcsname{\color{black}}%
      \expandafter\def\csname LT0\endcsname{\color{black}}%
      \expandafter\def\csname LT1\endcsname{\color{black}}%
      \expandafter\def\csname LT2\endcsname{\color{black}}%
      \expandafter\def\csname LT3\endcsname{\color{black}}%
      \expandafter\def\csname LT4\endcsname{\color{black}}%
      \expandafter\def\csname LT5\endcsname{\color{black}}%
      \expandafter\def\csname LT6\endcsname{\color{black}}%
      \expandafter\def\csname LT7\endcsname{\color{black}}%
      \expandafter\def\csname LT8\endcsname{\color{black}}%
    \fi
  \fi
    \setlength{\unitlength}{0.0500bp}%
    \ifx\gptboxheight\undefined%
      \newlength{\gptboxheight}%
      \newlength{\gptboxwidth}%
      \newsavebox{\gptboxtext}%
    \fi%
    \setlength{\fboxrule}{0.5pt}%
    \setlength{\fboxsep}{1pt}%
\begin{picture}(3888.00,3600.00)%
    \gplgaddtomacro\gplbacktext{%
      \csname LTb\endcsname
      \put(814,704){\makebox(0,0)[r]{\strut{}$10^{0}$}}%
      \csname LTb\endcsname
      \put(814,1239){\makebox(0,0)[r]{\strut{}$10^{1}$}}%
      \csname LTb\endcsname
      \put(814,1774){\makebox(0,0)[r]{\strut{}$10^{2}$}}%
      \csname LTb\endcsname
      \put(814,2309){\makebox(0,0)[r]{\strut{}$10^{3}$}}%
      \csname LTb\endcsname
      \put(814,2844){\makebox(0,0)[r]{\strut{}$10^{4}$}}%
      \csname LTb\endcsname
      \put(814,3379){\makebox(0,0)[r]{\strut{}$10^{5}$}}%
      \csname LTb\endcsname
      \put(946,484){\makebox(0,0){\strut{}$2$}}%
      \csname LTb\endcsname
      \put(1434,484){\makebox(0,0){\strut{}$4$}}%
      \csname LTb\endcsname
      \put(1922,484){\makebox(0,0){\strut{}$6$}}%
      \csname LTb\endcsname
      \put(2410,484){\makebox(0,0){\strut{}$8$}}%
      \csname LTb\endcsname
      \put(2897,484){\makebox(0,0){\strut{}$10$}}%
      \csname LTb\endcsname
      \put(3385,484){\makebox(0,0){\strut{}$12$}}%
      \csname LTb\endcsname
      \put(3873,484){\makebox(0,0){\strut{}$14$}}%
    }%
    \gplgaddtomacro\gplfronttext{%
      \csname LTb\endcsname
      \put(209,2041){\rotatebox{-270}{\makebox(0,0){\strut{}Frequency}}}%
      \put(2409,154){\makebox(0,0){\strut{}GloptiPoly execution time [s]}}%
    }%
    \gplbacktext
    \put(0,0){\includegraphics{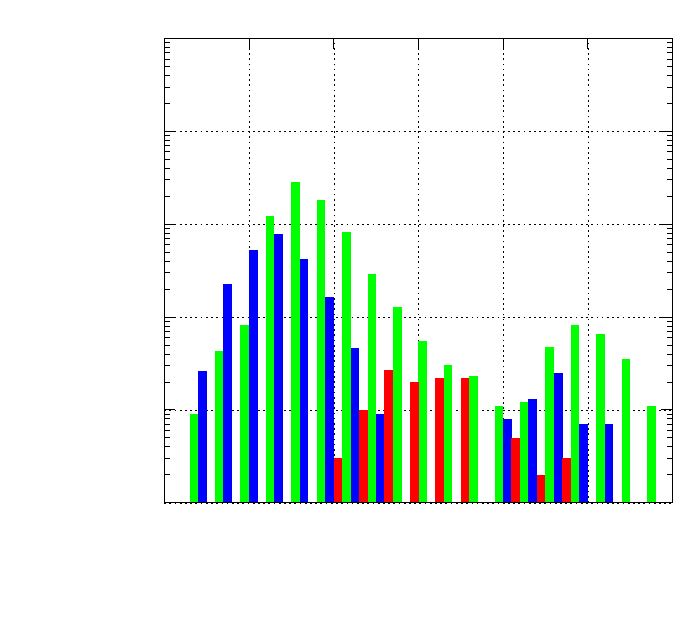}}%
    \gplfronttext
  \end{picture}%
\endgroup
  }%
  \resizebox{0.44\textwidth}{!}{
\begingroup
  \makeatletter
  \gdef\gplbacktext{}%
  \gdef\gplfronttext{}%
  \makeatother
  \ifGPblacktext
    \def\colorrgb#1{}%
    \def\colorgray#1{}%
  \else
    \ifGPcolor
      \def\colorrgb#1{\color[rgb]{#1}}%
      \def\colorgray#1{\color[gray]{#1}}%
      \expandafter\def\csname LTw\endcsname{\color{white}}%
      \expandafter\def\csname LTb\endcsname{\color{black}}%
      \expandafter\def\csname LTa\endcsname{\color{black}}%
      \expandafter\def\csname LT0\endcsname{\color[rgb]{1,0,0}}%
      \expandafter\def\csname LT1\endcsname{\color[rgb]{0,1,0}}%
      \expandafter\def\csname LT2\endcsname{\color[rgb]{0,0,1}}%
      \expandafter\def\csname LT3\endcsname{\color[rgb]{1,0,1}}%
      \expandafter\def\csname LT4\endcsname{\color[rgb]{0,1,1}}%
      \expandafter\def\csname LT5\endcsname{\color[rgb]{1,1,0}}%
      \expandafter\def\csname LT6\endcsname{\color[rgb]{0,0,0}}%
      \expandafter\def\csname LT7\endcsname{\color[rgb]{1,0.3,0}}%
      \expandafter\def\csname LT8\endcsname{\color[rgb]{0.5,0.5,0.5}}%
    \else
      \def\colorrgb#1{\color{black}}%
      \def\colorgray#1{\color[gray]{#1}}%
      \expandafter\def\csname LTw\endcsname{\color{white}}%
      \expandafter\def\csname LTb\endcsname{\color{black}}%
      \expandafter\def\csname LTa\endcsname{\color{black}}%
      \expandafter\def\csname LT0\endcsname{\color{black}}%
      \expandafter\def\csname LT1\endcsname{\color{black}}%
      \expandafter\def\csname LT2\endcsname{\color{black}}%
      \expandafter\def\csname LT3\endcsname{\color{black}}%
      \expandafter\def\csname LT4\endcsname{\color{black}}%
      \expandafter\def\csname LT5\endcsname{\color{black}}%
      \expandafter\def\csname LT6\endcsname{\color{black}}%
      \expandafter\def\csname LT7\endcsname{\color{black}}%
      \expandafter\def\csname LT8\endcsname{\color{black}}%
    \fi
  \fi
    \setlength{\unitlength}{0.0500bp}%
    \ifx\gptboxheight\undefined%
      \newlength{\gptboxheight}%
      \newlength{\gptboxwidth}%
      \newsavebox{\gptboxtext}%
    \fi%
    \setlength{\fboxrule}{0.5pt}%
    \setlength{\fboxsep}{1pt}%
\begin{picture}(3310.00,3600.00)%
    \gplgaddtomacro\gplbacktext{%
      \csname LTb\endcsname
      \put(-119,704){\makebox(0,0)[r]{\strut{}}}%
      \csname LTb\endcsname
      \put(-119,1239){\makebox(0,0)[r]{\strut{}}}%
      \csname LTb\endcsname
      \put(-119,1774){\makebox(0,0)[r]{\strut{}}}%
      \csname LTb\endcsname
      \put(-119,2309){\makebox(0,0)[r]{\strut{}}}%
      \csname LTb\endcsname
      \put(-119,2844){\makebox(0,0)[r]{\strut{}}}%
      \csname LTb\endcsname
      \put(-119,3379){\makebox(0,0)[r]{\strut{}}}%
      \csname LTb\endcsname
      \put(255,484){\makebox(0,0){\strut{}$2$}}%
      \csname LTb\endcsname
      \put(738,484){\makebox(0,0){\strut{}$2.2$}}%
      \csname LTb\endcsname
      \put(1221,484){\makebox(0,0){\strut{}$2.4$}}%
      \csname LTb\endcsname
      \put(1705,484){\makebox(0,0){\strut{}$2.6$}}%
      \csname LTb\endcsname
      \put(2188,484){\makebox(0,0){\strut{}$2.8$}}%
      \csname LTb\endcsname
      \put(2671,484){\makebox(0,0){\strut{}$3$}}%
    }%
    \gplgaddtomacro\gplfronttext{%
      \csname LTb\endcsname
      \put(1463,154){\makebox(0,0){\strut{}Maple execution time [s]}}%
      \put(1463,3269){\makebox(0,0){\strut{}}}%
      \csname LTb\endcsname
      \put(2388,3206){\makebox(0,0)[r]{\strut{}Feasible poses}}%
      \csname LTb\endcsname
      \put(2388,2986){\makebox(0,0)[r]{\strut{}Infeasible poses}}%
      \csname LTb\endcsname
      \put(2388,2766){\makebox(0,0)[r]{\strut{}Poses failed to compute}}%
    }%
    \gplbacktext
    \put(0,0){\includegraphics{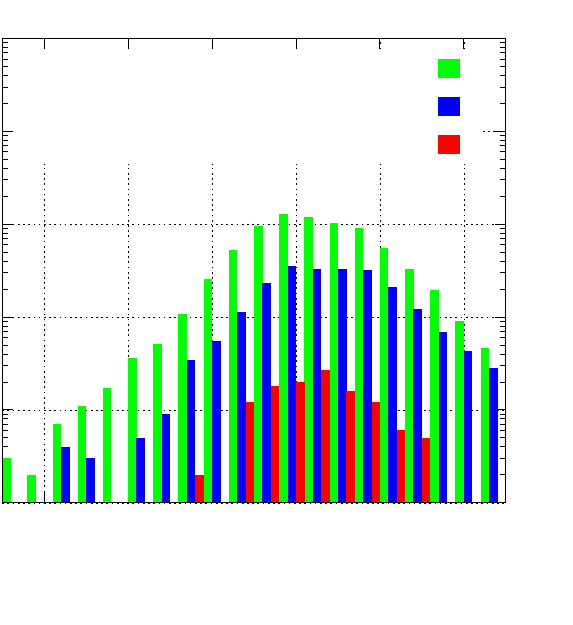}}%
    \gplfronttext
  \end{picture}%
\endgroup
  }
  \caption{Histograms of execution time. Left: execution time of the on-line phase of GloptiPoly. Right: execution time of the symbolic reduction and elimination in Maple.}
  \labelfig{Ill:GB:times}
\end{figure}

\section{Conclusions}
\noindent We presented a practical method for globally solving the 7DOF IK problem with a polynomial objective function. Our solution is accurate and can solve/decide infeasibility in $99~\%$ cases out of \num{10000} cases tested on the KUKA LBR IIWA manipulator. The code is open-sourced at \url{https://github.com/PavelTrutman/Global-7DOF-IKT}.

For future work, we consider two interesting directions. First, it is desirable to return a certificate of infeasibility when POP constraints are incompatible, e.g., by computing a SOS representation for the polynomial -1 on the quadratic module corresponding to the feasible set~\cite{ks13}. Secondly, it is interesting to exploit the structure of our POP to prove the  exactness of the observed second SDP relaxation in the moment-SOS hierarchy.
\hide{
In the case that the POP constraints are incompatible (i.e.\ the feasible set of admissible parameters is empty), it would be desirable to return a certificate of infeasibility. This certificate can be either numerical (obtained by solving the moment-SOS hierarchy with an SDP solver) or symbolic (obtained by Gr\"obner basis methods). It can be obtained e.g.\ by computing an SOS representation for the polynomial -1 (or any other negative polynomial) on the quadratic module corresponding to the feasible set, see e.g.\ \cite{ks13} in the specific case of certifying emptyness of spectrahedra (SDP feasibility sets).

It would be interesting to exploit the specific structure of the POP studied in this paper to prove (maybe under some assumptions on the data) exactness of the first or the second SDP relaxation in the moment-SOS hierarchy, i.e.\ that solving this relaxation always solves the original POP. For Euclidean distance POP arising in computer vision, this was achieved in \cite{aat12} by arguing on the curvature properties of the Lagrangian and its SOS representation in the quadratic module. 
}
\section*{Acknowledgments}
\noindent P.~Trutman was supported by the EU Structural and Investment Funds, Operational Programe Research, Development and Education under the project IMPACT (reg.\ no.\ CZ$.02.1.01/0.0/0.0/15\_003/0000468$) and Grant Agency of the CTU Prague project SGS19/173/OHK3/3T/13.
T.~Pajdla was supported by EU Structural and Investment Funds, Operational Programe Research, Development and Education under the project The Robotics for Industry 4.0 project (reg.\ no.\ CZ$.02.1.01/0.0/0.0/15\_003/0000470$, the EU H2020 ARtwin No.~856994, and EU H2020 SPRING No.~871245 Projects.
Didier Henrion and Mohab Safey El Din are supported by the European Union’s Horizon2020 research and innovation programme under the Marie Skłodowska-Curie grant agreement N°813211 (POEMA). 
Mohab Safey El Din is supported by the ANR grants ANR-18-CE33-0011 Sesame, ANR-19-CE40-0018 De Rerum Natura, ANR-19-CE48-0015 ECARP and the CAMiSAdo PGMO project.
\bibliographystyle{plain}
\bibliography{citations,Pajdla}
\end{document}